\documentclass{article}
\usepackage{amsmath}
\usepackage{amssymb}
\usepackage{amsthm} 
\usepackage{graphicx} 
\usepackage{geometry} 
\usepackage{xurl}

\usepackage[ruled,vlined]{algorithm2e}
\SetKwInput{KwResult}{Output}

\newtheorem{proposition}{Proposition}
\newtheorem{conjecture}{Conjecture}
\newtheorem{remark}{Remark}
\newtheorem{theorem}{Theorem}

\title{The Condition Number as a Scale-Invariant Proxy for Information Encoding in Neural Units}
\author{Oswaldo Ludwig}
\date{}

\begin{document}

\maketitle

\begin{abstract}
This paper explores the relationship between the condition number of a neural network's weight tensor and the extent of information encoded by the associated processing unit, viewed through the lens of information theory. It argues that a high condition number, though not sufficient for effective knowledge encoding, may indicate that the unit has learned to selectively amplify and compress information. This intuition is formalized for linear units with Gaussian inputs, linking the condition number and the transformation's log-volume scaling factor to the characteristics of the output entropy and the geometric properties of the learned transformation. The analysis demonstrates that for a fixed weight norm, a concentrated distribution of singular values (high condition number) corresponds to reduced overall information transfer, indicating a specialized and efficient encoding strategy. Furthermore, the linear stage entropy bound provides an upper limit on post-activation information for contractive, element-wise nonlinearities, supporting the condition number as a scale-invariant proxy for encoding capacity in practical neural networks. An empirical case study applies these principles to guide selective fine-tuning of Large Language Models for both a new task and a new input modality. The experiments show that the proposed method, named KappaTune, effectively mitigates catastrophic forgetting. Unlike many existing catastrophic forgetting mitigation methods that rely on access to pre-training statistics, which are often unavailable, this selective fine-tuning approach offers a way to bypass this common requirement.
\end{abstract}

\vspace{1em}

\section{Introduction}
Pre-trained neural networks have demonstrated remarkable success in various domains, largely due to their ability to encode rich, transferable information within their weight tensors. Understanding how this information is encoded at the level of individual units or layers is crucial for interpretability and further model development. This work investigates the use of the weight tensor's condition number as a proxy for the significance of information encoding within a processing unit, especially when considering information-theoretic concepts like entropy.

A major challenge in adapting pre-trained models to new tasks or modalities is catastrophic forgetting (CF), where the model rapidly loses previously acquired knowledge. Numerous methods have been proposed to mitigate CF, generally falling into two broad categories: those that require access to pre-training statistics or old data, and those that do not.

The first category includes approaches such as experience replay \cite{lopez2017gradient}, which re-trains on a small portion of previous data, and knowledge distillation \cite{li2017learning}, which uses the previous model's outputs as soft targets. Regularization techniques like Elastic Weight Consolidation (EWC) \cite{kirkpatrick2017overcoming} and Synaptic Intelligence (SI) \cite{zenke2017continual} also often fall into this category, as they may rely on calculating parameter importance with respect to past tasks' data. While often effective, the practicality of these methods is significantly limited by their critical dependence on data from prior tasks, which is frequently unavailable due to privacy concerns, storage limitations, or access restrictions.

The second category comprises methods that do not require explicit access to pre-training statistics. These strategies typically focus on intrinsic model properties or architectural adjustments. Examples include architectural methods like Progressive Neural Networks (PNNs) \cite{rusu2016progressive} and expert gateways \cite{rosenbaum2018routing}, as well as various regularization or parameter isolation techniques. Examples of the latter include Memory Aware Synapses (MAS) \cite{aljundi2018memory} and PackNet \cite{mallya2018packnet}. These methods attempt to preserve old knowledge by selectively modifying or freezing parts of the model without explicit access to past data.

More recently, we proposed CLASSP (Continual Learning through Adjustment Suppression and Sparsity Promotion) \cite{ludwig2024classp}, a biologically-inspired approach that mitigates forgetting by combining a decay rate over weight adjustment, i.e. suppressing updates for frequently updated weights, with a threshold on the loss gradient that promotes sparse learning by updating only weights with significant impact.

Beyond theoretical exposition, we present a practical case study where the principles of information encoding, particularly through the lens of the condition number, are applied to guide the selective fine-tuning of a multimodal Large Language Model. This application directly addresses the critical challenge of catastrophic forgetting when adapting pre-trained models to new modalities, offering a novel approach that falls into the category of methods not requiring access to prior task statistics. This demonstrates how an understanding of a tensor's information-encoding characteristics enables more robust and efficient model adaptation strategies or even model pruning \cite{ludwig2023compressing}.

The remainder of this paper is organized as follows. Section \ref{sec:condition_number} introduces the concept of the condition number and its relevance beyond numerical stability. Section \ref{sec:information_encoding} delves into the theoretical foundations of information encoding, discussing differential entropy, the importance of scale-invariant measures, and formally proving how singular value distribution impacts information volume. Section \ref{sec:case_study} presents a practical case study on selective fine-tuning of Large Language Models, detailing the model setup, training algorithm, and the connection of our approach to information-theoretic principles for mitigating catastrophic forgetting. Finally, Section \ref{sec:conclusion} concludes the paper.

\section{From Stability to Semantics: Rethinking the Condition Number}
\label{sec:condition_number}

Beyond their well-established role in numerical stability and optimization, the singular values and condition number of neural network weight matrices have also become a focal point for understanding model behavior, generalization, and information processing. Prior work has shown that a flatter singular value spectrum can promote better generalization by enabling more uniform transformations and reducing sensitivity to input perturbations \cite{jia2017improving,ludwig2014eigenvalue}. Other studies have explored how singular values relate to information compression and flow through network layers, aligning with the principles of the Information Bottleneck framework \cite{tishby2015deep}.

In contrast to these approaches, which primarily focus on activations or global transformations, our work leverages the condition number of the weight tensor as a direct, intrinsic measure of the significance of information encoding within individual units.

The condition number $\kappa(W)$ of a matrix $W$ quantifies the sensitivity of the output of a linear system to small changes in the input. For a weight matrix $W \in \mathbb{R}^{m \times n}$, it's defined as the ratio of its largest singular value ($\sigma_{\max}(W)$) to its smallest non-zero singular value ($\sigma_{\min}(W)$):
\begin{equation} \label{eq:condition_number}
\kappa(W) = \frac{\sigma_{\max}(W)}{\sigma_{\min}(W)}
\end{equation}
A high condition number indicates numerical instability, where small perturbations in the input can lead to disproportionately large changes in the output. In neural networks, this can affect training stability and robustness. However, we propose that beyond numerical stability, the condition number also carries information about the nature of the learned transformation.

\section{Information Encoding and Entropy}
\label{sec:information_encoding}

Information theory provides tools to quantify the uncertainty and information content within data. Shannon's differential entropy $h(Z)$ for a continuous random vector $Z$ with probability density function $f_Z(z)$ is given by:
\begin{equation} \label{eq:differential_entropy}
h(Z) = - \int f_Z(z) \log_2 f_Z(z) dz
\end{equation}

Consider a linear unit where the output $Y$ is given by $Y = WX$, with $W \in \mathbb{R}^{m \times n}$ being the weight matrix and $X \in \mathbb{R}^n$ being a continuous random input vector. For a multivariate Gaussian input $X \sim \mathcal{N}(\mu_X, \Sigma_X)$, the output $Y$ follows $Y \sim \mathcal{N}(W\mu_X, W\Sigma_X W^T)$.
If we assume a spherical input covariance $\Sigma_X = \sigma_x^2 I_n$, then the output covariance is $\Sigma_Y = \sigma_x^2 W W^T$. Using the general formula for the differential entropy of a multivariate Gaussian distribution, the differential entropy of the output $Y$ is:
\begin{equation} \label{eq:output_entropy_det}
h(Y) = \frac{1}{2} \log_2((2\pi e)^m \det(\Sigma_Y))
\end{equation}
Substituting:
\begin{equation} \label{eq:det}
\det(\Sigma_Y) = \det(\sigma_x^2 W W^T) = (\sigma_x^2)^m \det(W W^T)
\end{equation}
and considering that:
\begin{equation} \label{eq:det2}
\det(W W^T) = (\prod_{i=1}^{\min(m,n)} \sigma_i(W))^2
\end{equation}
we can formulate:
\begin{equation} \label{eq:output_entropy_singular_values}
h(Y) = \frac{m}{2} \log_2(2\pi e \sigma_x^2) + \sum_{i=1}^{\min(m,n)} \log_2(\sigma_i(W))
\end{equation}
Equation \eqref{eq:output_entropy_singular_values} highlights the direct relationship between the output entropy and the product of the singular values of the weight matrix. From an information-theoretic perspective, a smaller entropy $h(Y)$ signifies less uncertainty about the output vector $Y$. For a neural network unit, this reduction in uncertainty is a key aspect of knowledge encoding, i.e. the unit learns to transform the potentially high-entropy (uncertain) input $X$ into a lower-entropy (more certain and predictable) representation $Y$ that highlights discriminative features relevant to the task.

\subsection{Geometric Foundations and Scale-Invariance}

Consider the aforementioned linear transformation $Y = WX$, where $W \in \mathbb{R}^{m \times n}$. The singular values $\sigma_i$ of $W$ govern the axis-wise scaling of the transformation. While the Frobenius norm $\|W\|_F$ reflects the overall scaling of the transformation, the condition number $\kappa(W) = \sigma_{\max}/\sigma_{\min}$ (for full-rank $W$) characterizes the anisotropy of the transformation and is invariant to uniform scaling of $W$.

This scale-invariant geometric property represents the unit's directional selectivity. A unit with $\kappa(W) \approx 1$ preserves the geometry of the input space, whereas a unit with high $\kappa(W)$ prioritizes a specific subspace while suppressing others. As shown in the next subsections, this anisotropy becomes a key determinant of the unit's information encoding behavior under fixed-norm constraints.

\subsection{A Proxy for Knowledge Encoding: Scale-Invariant Perspective}
Building upon the understanding that the condition number $\kappa(W)$ offers scale-invariant insights into the properties of the linear transformation $W$, we now interpret its role in encoding significant pre-trained information. The following propositions and analyses are thus grounded in these scale-invariant measures of information encoding.

\begin{proposition}
For a linear transformation $Y=WX$, where $X$ is a random vector with non-zero entropy, a high condition number $\kappa(W) \gg 1$ implies an anisotropic (oblong) transformation. This anisotropy indicates that the unit has learned to selectively amplify certain input directions (those corresponding to large singular values) while compressing others (those corresponding to small singular values). This selective processing leads to a restructuring of the information distribution in the output: information about discriminative features is preserved or amplified, while information about irrelevant or redundant features is attenuated, resulting in a more efficient encoding of knowledge. This efficiency is reflected in the structural properties of the output distribution, which carries concentrated, relevant information as quantified by the interplay between the condition number and the log-volume scaling factor.
\end{proposition}

\begin{remark}
Note that a high condition number is not a sufficient condition for effective knowledge encoding. It can also arise from numerical ill-conditioning due to poor training or inherent redundancy in the learned mapping, which may not correspond to semantically meaningful feature extraction. However, for a well-trained model, this anisotropy is a desirable trait, as it optimizes the transfer of truly discriminative information.
\end{remark}

\begin{conjecture}
Assuming a well-trained model, a high condition number $\kappa(W) \gg 1$ is suggested to be a necessary (though not sufficient) condition for a weight tensor $W$ to have encoded significant discriminative pre-training information in a robust and desirable sense. This conjecture warrants further empirical and theoretical investigation, particularly in the context of non-linear transformations and real-world data distributions.
\end{conjecture}

To further illustrate the impact of singular value distribution on information encoding, we consider a scenario where the Frobenius norm of the weight matrix $W$ is fixed.

\begin{theorem}[Maximum Differential Entropy under Frobenius Norm Constraint]
Let $W \in \mathbb{R}^{n \times m}$ be a matrix with singular values $\sigma_1, \dots, \sigma_p$, where $p = \min(n, m)$. Assuming that the Frobenius norm of $W$ is fixed, $\sum_{i=1}^n \sum_{j=1}^m W_{ij}^2 = \sum_{k=1}^{p} \sigma_k^2 = C$, the differential entropy of a multivariate Gaussian transformed by $W$ is maximized when the matrix $W$ has condition number $\kappa(W) = 1$.
\end{theorem}

\begin{proof}
As can be seen in (\ref{eq:output_entropy_singular_values}), maximizing entropy is therefore equivalent to maximizing the log-volume scaling factor
\begin{equation}
f(\sigma_1, \dots, \sigma_p) = \sum_{k=1}^{p} \log_2(\sigma_k)
\end{equation}
But we also want it subject to the constraint
\begin{equation}
g(\sigma_1, \dots, \sigma_p) = \sum_{k=1}^{p} \sigma_k^2 - C = 0.
\end{equation}
Using the method of Lagrange multipliers, we define the Lagrangian:
\begin{equation}
\mathcal{L}(\sigma_1, \dots, \sigma_p, \lambda) = \sum_{k=1}^{p} \log_2(\sigma_k) - \lambda \left( \sum_{k=1}^{p} \sigma_k^2 - C \right).
\end{equation}
Taking partial derivatives with respect to each \( \sigma_k \) and setting them to zero:
\begin{equation}
\frac{\partial \mathcal{L}}{\partial \sigma_k} = \frac{1}{\sigma_k \ln 2} - 2\lambda\sigma_k = 0.
\end{equation}
Solving for \( \sigma_k \), we obtain:
\begin{equation}
\sigma_k^2 = \frac{1}{2\lambda \ln 2}.
\end{equation}
Since the right-hand side is independent of $k$, it follows that all $\sigma_k$ must be equal. In this case, the condition number $\kappa(W) = \frac{\sigma_{\max}(W)}{\sigma_{\min}(W)} = 1$ (assuming non-zero singular values, which is typical for an active channel). Because the objective function is strictly concave over its domain of positive singular values and the constraint set is convex, this critical point represents the unique global maximum.
\end{proof}

As proven, a uniform distribution of singular values maximizes the log-volume scaling factor for a given Frobenius norm\footnote{The differential entropy $h(Z)$ is a scale-variant measure. For a random vector $Z\in\mathbb{R}^m$ and constant $c\neq0$, $h(cZ) = h(Z) + \log_2|c|$}, leading to increased information transfer with higher output entropy and more uncertainty about the output vector. Conversely, tensors with non-uniform singular value distributions (higher condition numbers) exhibit anisotropic behavior, selectively amplifying certain directions while compressing others. In well-trained models, this anisotropy reflects learned feature hierarchies that concentrate representational capacity on task-relevant signals while attenuating less predictive input variations, reducing output entropy while preserving predictive power. This perspective aligns with the Information Bottleneck principle \cite{tishby2015deep}, where effective representations compress input information, reducing mutual information $I(X;Z)$ while retaining task-relevant features, maximizing mutual information $I(Z;Y)$, where $X$ denotes the input data, $Z$ the compressed learned representation, and $Y$ the target output.

Geometrically, a uniform distribution corresponds to a nearly spherical transformation that treats all input directions equally, preserving a broader range of input variability and indicating less specialized encoding. Deviation from uniformity (concentrated singular values, high condition number) yields a smaller log-volume scaling factor for the same Frobenius norm, manifesting as an oblong transformation in which the unit compresses or discards information from less important directions to focus on discriminative features. In the extreme, a highly rank-deficient matrix (most singular values near zero) collapses the input space to a lower-dimensional subspace, signifying ultimate loss of input variability and limited information transfer.

In summary, Theorem 1 establishes, under the given assumptions, that tensors with low condition numbers maximize differential entropy, positioning them as the most information-rich and potentially least specialized components. This theoretical foundation explains their suitability for adaptation tasks, where broad representational capacity facilitates efficient learning without disrupting pre-existing information encoding.

\begin{remark}[Information-Theoretic upper bound for Nonlinear Units]
Let $Z = Wx$ denote the linear pre-activation output, where $x \in \mathbb{R}^n$ is the input vector and $W \in \mathbb{R}^{m \times n}$ is the weight matrix. For a differentiable activation function $\phi$ applied element-wise, the change of variables formula for differential entropy states that:
\begin{equation}
\label{entropy1}
    h(\phi(Z)) = h(Z) + \mathbb{E}[\log |\det J_\phi(Z)|]
\end{equation}
where $J_\phi$ is the Jacobian of $\phi$. For common activations such as leaky ReLU, softplus, tanh, or sigmoid, the derivative satisfies $|\phi'(z)| \leq 1$ almost\footnote{For certain activation functions (e.g., ReLU at $z=0$), the derivative is not strictly defined at isolated points, but this does not affect the entropy calculation since such points have measure zero.} everywhere.
Since $\phi$ acts coordinate-wise, for any vector $z \in \mathbb{R}^m$, the Jacobian matrix is diagonal:
\[
J_\phi(z) = \mathrm{diag}\big(\phi'(z_1),\ldots,\phi'(z_m)\big),
\]
so
\[
\det J_\phi(z) = \prod_{i=1}^m \phi'(z_i)
\quad\text{and}\quad
\log\big|\det J_\phi(z)\big|=\sum_{i=1}^m \log\big|\phi'(z_i)\big|.
\]
If $|\phi'(z_i)|\le 1$ for almost all $z_i$, then $\log|\phi'(z_i)|\le 0$ for each $i$, hence
$\log|\det J_\phi(z)|\le 0$ pointwise. Taking expectations over $Z$ preserves the inequality:
\begin{equation}
\label{entropy2}
\mathbb{E}\big[\log|\det J_\phi(Z)|\big]=\sum_{i=1}^m \mathbb{E}\big[\log|\phi'(Z_i)|\big]\le 0,
\end{equation}
Substituting (\ref{entropy2}) into (\ref{entropy1}) yields $h(\phi(Z)) \leq h(Z)$. Consequently, for any weight matrix $W \in \mathbb{R}^{m\times n}$, the maximum linear entropy $h(Z)\big|_{\kappa(W)=1}$ \emph{under a fixed Frobenius-norm constraint on $W$} (as established in Theorem~1 via the singular-value formulation in (\ref{eq:output_entropy_singular_values})) serves as a theoretical upper bound for the information content available to the post-activation stage. The condition number $\kappa(W)$ thus acts as a scale-invariant proxy for the reduction in this potential capacity, with implications that propagate through subsequent nonlinear transformations.
\end{remark}

\section{Case Study: Selective Fine-Tuning for Catastrophic Forgetting Mitigation}
\label{sec:case_study}

Catastrophic forgetting poses a critical challenge when adapting large pre-trained models to new tasks or modalities. In multimodal LLMs (MM-LLMs), for instance, adding audio processing while preserving extensive linguistic knowledge is essential. Our hypothesis is that insights from the condition number and singular value distribution of weight tensors can guide which parts of a pre-trained LLM to fine-tune, enabling adaptation while mitigating forgetting. We instantiate this idea in \textbf{KappaTune}\footnote{https://github.com/oswaldoludwig/kappaTune}, a training algorithm that selectively unfreezes low-kappa tensors.

\subsection{KappaTune: General Algorithm}
\label{sec:kappatune_general}

KappaTune selectively fine-tunes a pre-trained model by unfreezing a small budget of low-kappa tensors, i.e. those with the smallest condition numbers, while keeping the remaining parameters frozen. This concentrates adaptation on less specialized, numerically stable components to mitigate catastrophic forgetting.

Let $W$ denote a weight tensor. Multi-dimensional tensors (e.g., convolutional kernels) are reshaped into a 2D matrix $(d_{\text{out}}, d_{\text{in}})$ to enable singular value decomposition (SVD). KappaTune precomputes $\kappa(W) = \sigma_{\max}(W)/\sigma_{\min}(W)$ for all eligible tensors (excluding biases and normalization scalars unless stated) and unfreezes a fixed budget of those with the smallest condition numbers (low anisotropy). The resulting set of tensor names is stored and used to selectively unfreeze during training, see Algorithm \ref{alg:kappa_general} for further details.

{\spaceskip=0.1em \relax
\begin{algorithm}[H]
\scriptsize
\SetAlgoLined
\KwData{Task dataset $\mathcal{D}$ (inputs $x$, targets $y$), pre-trained model $M$, eligible tensor set $\mathcal{S}$, budget $K$}\
\KwResult{Partially fine-tuned model $M$}\
\textbf{Precompute selection:}\\
Define reshape rule per layer type; compute $\kappa(W)=\sigma_{\max}(W)/\sigma_{\min}(W)$ for all $W\in\mathcal{S}$\;
Sort $\mathcal{S}$ by $\kappa$ ascending; $T_{\text{trainable}} \gets$ first $K$ tensor names\;
\textbf{Initialize:}\\
Freeze all parameters of $M$; unfreeze $t\in T_{\text{trainable}}$\;
Optimizer $\leftarrow$ Adam($\left\{t\in T_{\text{trainable}}\right\}$); choose task loss (e.g., cross-entropy)\;
\For{\textbf{epoch} $=1$ \KwTo NumEpochs}{
  \For{\textbf{batch} $(x,y)\in\mathcal{D}$}{
    \textbf{Forward:} $o \gets M(x)$ \tcp*{Task-agnostic forward}\
    \textbf{Loss:} $\mathcal{L}(o,y)$ \tcp*{Causal LM / classification / seq2seq}\
    \textbf{Backward \& Update:} \texttt{backward(); step(); zero\_grad()}.
  }
  Save checkpoint: $M.\texttt{save\_pretrained()}$ (and any task-specific heads).\
}
\caption{KappaTune (general selective fine-tuning)}
\label{alg:kappa_general}
\end{algorithm}
}

\subsection{Adapting a pre-trained LLM to a new task}

KappaTune is assessed on \emph{OLMoE} \cite{muennighoff2024olmoe}, an open-source MoE-based LLM (1B active parameters, 7B total), configured with 64 experts per layer and top-8 expert activation. The motivation for using an MoE model is the finer granularity for tensor selection in modular experts\footnote{\url{https://github.com/oswaldoludwig/kappaTune/blob/main/experiments_sarcasm_kappatune.py}}. 

For task adaptation, we use the small \emph{TweetEval–Irony} subset with 500 training and 200 test samples \cite{barbieri2020tweetevalunifiedbenchmarkcomparative}. Forgetting is measured post hoc via the change in perplexity ($\Delta$PPL) on WikiText-2, a diverse corpus for language modeling benchmarks that serves as the baseline control for evaluating general retention. $\Delta$PPL closer to zero indicates near-perfect retention of pre-trained knowledge.

OLMoE is fine-tuned as a prompted classifier, fine-tuning it to output exact "yes" or "no" tokens following the prompt "Tweet: \{text\}. Sarcasm:", with evaluation based on the logit values for these tokens. This forces the model to conform to a strict classification format, which the base OLMoE is not pretrained for, resulting in a modest performance, even like this the relative improvements of KappaTune over LoRA can be clearly observed.

KappaTune unfreezes 75 low-kappa tensors resulting $\sim$153M trainable parameters across the MoE's attention and MLP modules (experts and shared blocks), while leaving routing behavior unchanged. Selection and optimization follow Algorithm~\ref{alg:kappa_general}. 

The LoRA baseline matches a similar number of trainable parameters ($\sim$153M) with LoRA rank of 16 applied to attention projections (\texttt{q\_proj}, \texttt{v\_proj}) and MoE expert projections (\texttt{gate\_proj}, \texttt{up\_proj}, \texttt{down\_proj}).

\begin{table}[htbp]
  \centering
  \caption{KappaTune vs. LoRA on OLMoE. All metrics (Accuracy, $\Delta$PPL, Final Train Loss, and Wall-clock time) represent mean values averaged over 10 experiments.}
  \label{tab:moe_results_sorted}
  \begin{tabular}{|l|c|c|c|c|c|}
    \hline
    Method & Accuracy (\%) & Forgetting ($\Delta$PPL) & Final Train Loss & Train Time (s) & Epochs\\
    \hline
    KappaTune & 70.33 & 0.2363 & 2.8902 & 740 & 18\\
    \hline
    LoRA      & 70.67 & 1.4783 & 2.8893 & 1271 & 12\\
    \hline
  \end{tabular}
\end{table}

KappaTune significantly outperforms the LoRA baseline in forgetting mitigation when both methods are trained to achieve comparable accuracy levels, exhibiting minimal degradation with $\Delta$PPL = 0.2363, over six times less than LoRA's $\Delta$PPL = 1.4783. This highlights KappaTune's superior preservation of pre-trained knowledge despite matched parameter budgets. The efficiency gains are further evident in faster wall-clock training times, as LoRA required fewer epochs but incurred higher overhead from adapters compared to KappaTune's direct tensor unfreezing\footnote{The reported training time excludes the one-time selection overhead of approximately 370s (measured on CPU) required to calculate condition numbers for all candidate tensors.}.

\subsection{Adapting a pre-trained LLM to a new task via KappaTune-LoRA}
\label{sec:imdb_study}

LoRA adapters enable task-specific fine-tuning while allowing attachment and detachment to avoid catastrophic forgetting. However, preserving pre-trained general knowledge, even when LoRA adapters are attached for a specific task, enhances the model's reasoning capabilities on the target task, thereby improving generalization capacity. To assess this in a high-parameter regime, we perform a comparative analysis on the DeepSeek-V2-Lite LLM \cite{liu2024deepseek}, a MoE model with 16B total parameters and $\sim$2.4B active parameters. 

Specifically, we evaluate whether KappaTune-LoRA, which strategically places LoRA adapters on tensors with low condition numbers, outperforms standard uniform placement. This comparison controls the trainable parameters by approximately 190 million for both experiments, focusing on generalization in sentiment analysis\footnote{\url{https://github.com/oswaldoludwig/kappaTune/blob/main/experiments_SA_kappatune.py}}.

The model is loaded in 4-bit precision and two adaptation strategies are adopted:
\begin{itemize}
    \item \textbf{LoRA Global:} Standard low-rank adaptation ($r=16$, $\alpha=32$) targeting all major linear projections (\texttt{q\_proj}, \texttt{v\_proj}, \texttt{up\_proj}, \texttt{down\_proj}). This configuration serves as the dense adaptation baseline.
    \item \textbf{KappaTune-LoRA:} The condition number ($\kappa$) is computed for all expert tensors in the MoE layers. A budget of $K=300$ tensors with the lowest $\kappa$ is selected. To ensure a fair comparison of capacity, the rank is increased to $r=190$ for these selected adapters, resulting in the same number of trainable parameters as LoRA Global. 
\end{itemize}

Both methods are trained on the IMDB sentiment dataset \cite{maas2011learning} and evaluated on WikiText-2 to assess catastrophic forgetting, while generalization capacity is measured on IMDB via the performance gap ($\Delta$PPL) between training and test data.

The IMDB data is formatted as ``Review: [text] \textbackslash n Sentiment: [positive/negative]'' for causal language modeling during training and evaluation, with the model's performance assessed via perplexity computed over these annotated sequences.

Table \ref{tab:imdb_deltas} presents the generalization gap ($\Delta$PPL) on IMDB, defined as the difference between test and train perplexity, and the forgetting cost ($\Delta$PPL) on WikiText-2, defined as the increase in perplexity after fine-tuning relative to the unadapted baseline. In both cases, a positive $\Delta$ indicates degradation, with smaller values signifying better performance.

\begin{table}[htbp]
  \centering
  \caption{Generalization and Forgetting Analysis. Lower values are better for both metrics.}
  \label{tab:imdb_deltas}
  \begin{tabular}{l|c||c}
    \hline
    \textbf{Method} & \textbf{Generalization Gap ($\Delta$PPL)} & \textbf{Forgetting Cost ($\Delta$PPL)} \\
     & ($\text{PPL}_{\text{IMDB Test}} - \text{PPL}_{\text{IMDB Train}}$) & ($\text{PPL}_{\text{Wiki Fine-Tuned model}} - \text{PPL}_{\text{Wiki Base model}}$) \\
    \hline
    LoRA Global & +2.06 & +3.44 \\
    KappaTune & \textbf{+1.61} & \textbf{+1.21} \\
    \hline
  \end{tabular}
\end{table}

The results indicate a link between knowledge retention and task generalization. When trained to achieve equivalent performance on the IMDB sentiment classification task as KappaTune-LoRA, the LoRA Global baseline (with uniform adapter placement) exhibits 2.8$\times$ higher catastrophic forgetting and a 1.3$\times$ larger generalization gap, as measured by the increase in perplexity on held-out data sets.

This superior generalization supports our hypothesis that by restricting adaptation to stable and less specialized tensors, KappaTune preserves the model's broad linguistic capabilities (forgetting cost of only +1.21). This retained general knowledge acts as an inductive bias, aiding the model in reasoning about the new task rather than relying solely on surface-level pattern matching from the training data.

\subsection{Adapting a pre-trained LLM to a new input modality (ASR)}

The MM-LLM is composed with a frozen audio encoder and a text LLM connected by a trainable adaptor. The audio encoder is the Whisper large-v2 encoder (frozen) \cite{radford2022robustspeechrecognitionlargescale}, which transforms raw audio into high-dimensional embeddings. The text LLM is Llama-3.2-8B-Instruct \cite{llama3-herd}. The trainable adaptor bridges Whisper outputs to the Llama embedding space via a linear projection, layer normalization, and two multihead attention layers with positional encoding. Audio features are sub-sampled $4\times$ by concatenation.

The model is fine-tuned on proprietary ASR audio-text pairs (English and German) in the automotive domain, with details withheld due to confidentiality. We report Word Error Rate (WER) on held-out English and German test sets. All Whisper parameters are frozen. Only KappaTune-selected tensors are unfrozen. We train with Adam on the adaptor plus unfrozen Llama tensors, using token-level cross-entropy and shifted labels for causal prediction aligned to the concatenated \{adapted audio, BOS, text\} embeddings. Tensor selection and optimizer settings follow Algorithm~\ref{alg:kappa_general} and the multimodal forward and shifted labeling follow Algorithm~\ref{alg:kappa_mm}.

{\spaceskip=0.1em \relax
\begin{algorithm}[H]
\scriptsize
\SetAlgoLined
\KwData{Audio-text pairs $(X_{\text{audio}}, X_{\text{text}})$, frozen Whisper encoder $E_W$, Llama $M_L$, adaptor $A$, tokenizer $\mathcal{T}$, $T_{\text{trainable}}$ from Alg.~\ref{alg:kappa_general}}
\KwResult{Fine-tuned adaptor $A$, partially fine-tuned Llama $M_L$}

\textbf{Initialize (per Alg.~\ref{alg:kappa_general}):}\\
Freeze all $E_W$; freeze all $M_L$; unfreeze $t\in T_{\text{trainable}}$ in $M_L$;\\
Set optimizer on $\{A\}\cup\{t\in T_{\text{trainable}}\}$; set loss to token-level cross-entropy.\\

\For{\textbf{epoch} $=1$ \KwTo NumEpochs}{
  \For{\textbf{batch} $(\texttt{audio}, \texttt{text}, \texttt{len}_a, \texttt{len}_t)$}{
    $Y_{\text{whisper}} \gets E_W(\texttt{audio})$ \tcp*{Padded mel features $\rightarrow$ embeddings}
    $E_{\text{adapted}} \gets A(Y_{\text{whisper}})$ \tcp*{Linear + LN + MHA + PE; $4\times$ sub-sampling}
    $Tokens \gets \mathcal{T}(\texttt{text} + \texttt{EOS})$;\quad $E_{\text{labels}} \gets \text{emb}(Tokens)$;\quad $E_{\text{BOS}} \gets \text{emb}(\texttt{BOS})$\\
    $InputsEmbeds \gets \text{Concat}(E_{\text{adapted}}, E_{\text{BOS}}, E_{\text{labels}})$\\
    Build $CausalLabels$ aligned to $(\texttt{len}_a + 1 + \texttt{len}_t)$, masking non-predictive positions with $-100$\\
    $Loss \gets \text{CE}(M_L(\texttt{inputs\_embeds=InputsEmbeds}), \texttt{labels=CausalLabels})$\\
    \texttt{backward(); step(); zero\_grad()}.
  }
  Save checkpoints: $A.\texttt{save\_state()}$, $M_L.\texttt{save\_pretrained()}$.
}
\caption{KappaTune-MM (multimodal version, selection and optimization per Alg.~\ref{alg:kappa_general})}
\label{alg:kappa_mm}
\end{algorithm}
}

Figure~\ref{fig:wer_plot_new} reports average WER on English and German test data over training days for Llama-3.2-8B with two unfreezing strategies: the \emph{lowest-kappa} 100 tensors (\textit{8B LLM selected 100}) and the \emph{highest-kappa} 100 tensors (\textit{8B LLM INVERSE 100}). Both share identical hyperparameters and similar trainable parameter counts.

\begin{figure}[htbp]
\centering
\includegraphics[width=0.6\textwidth]{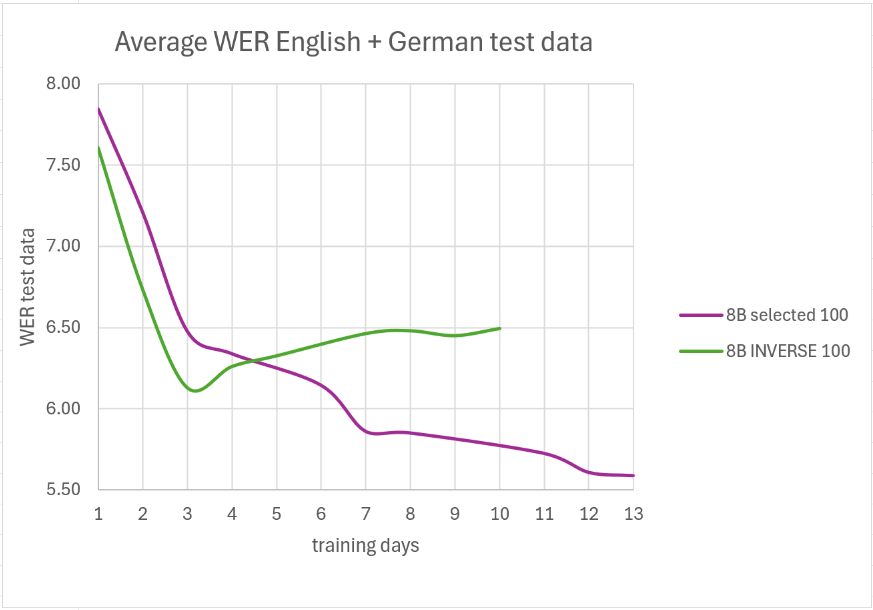}
\caption{Average WER (English \& German) for two Llama-8B multimodal configurations. We compare unfreezing 100 \emph{lowest}-kappa tensors vs.~100 \emph{highest}-kappa tensors.}
\label{fig:wer_plot_new}
\end{figure}

The \textit{INVERSE 100} strategy initially achieves 6.13\% WER at day~3 but degrades thereafter, consistent with catastrophic forgetting when highly anisotropic (high-kappa) tensors are updated. These tensors likely encode specialized filters; perturbing them overwrites critical pre-trained knowledge and may also introduce numerical instability. In contrast, \textit{selected 100} (lowest-kappa) monotonically reduces WER to 5.61\% by day~12, supporting the hypothesis that adapting less specialized, numerically stable tensors enables new capabilities with reduced forgetting.

\section{Conclusion}
\label{sec:conclusion}

This paper shows that the condition number of a weight tensor, while traditionally associated with numerical stability, also serves as an indicator of the information-processing strategy learned during pre-training. A high condition number reflects the unit's ability to perform efficient information compression and amplification, a critical aspect of encoding significant discriminative knowledge. This anisotropic transformation reshapes the input into a more focused and informative representation, which is central to the effectiveness of pre-trained models.

By analyzing measures such as the condition number and the log-volume scaling factor, we gain deeper insights into how neural transformations encode information. Our theoretical analysis, supported by a practical case study on LLM fine-tuning, demonstrates that these principles can guide selective adaptation strategies that mitigate catastrophic forgetting without requiring access to pre-training data.

Future work should extend this analysis beyond Gaussian assumptions to realistic data distributions, rigorously quantifying how the full singular value spectrum, not only the condition number, governs information encoding in deep networks. Additionally, broader empirical validation across architectures, tasks, and modalities, combined with adaptive threshold selection policies, would enhance the method's generality and scalability for large neural systems.

\bibliographystyle{IEEEtran}
\bibliography{biblio}

\end{document}